\RequirePackage[l2tabu, orthodox]{nag}
\documentclass[ pagesize, twoside=off, bibliography=totoc, DIV=calc, fontsize=11pt, a4paper, french, english]{scrartcl}
	\input glyphtounicode
	\usepackage{lmodern}
	\usepackage[T1]{fontenc}
	\usepackage[utf8]{inputenc}
	\usepackage{newunicodechar}
	\usepackage{textcomp}
	\usepackage{stmaryrd}\SetSymbolFont{stmry}{bold}{U}{stmry}{m}{n}
	\DeclareFontShape{OMX}{cmex}{bx}{n}{%
	   <->sfixed*cmexb10%
	   }{}
	\SetSymbolFont{largesymbols}{bold}{OMX}{cmex}{bx}{n}
	\usepackage{slantsc}
	\AtBeginDocument{%
	}
	\usepackage{booktabs}
	\usepackage{array}
        \usepackage{multicol}
        \usepackage{multirow}
	\usepackage{xpatch}
	\usepackage{amssymb}
	\usepackage{mathtools}
	\usepackage{natbib}

\newtoggle{LCpres}
	\newtoggle{LCart}
	\newtoggle{LCposter}
	\makeatletter
	\@ifclassloaded{beamer}{
		\toggletrue{LCpres}
		\togglefalse{LCart}
		\togglefalse{LCposter}
		\wlog{Presentation mode}
	}{
		\@ifclassloaded{tikzposter}{
			\toggletrue{LCposter}
			\togglefalse{LCpres}
			\togglefalse{LCart}
			\wlog{Poster mode}
		}{
			\toggletrue{LCart}
			\togglefalse{LCpres}
			\togglefalse{LCposter}
			\wlog{Article mode}
		}
	}
	\makeatother%

	\usepackage{babel}
		\usepackage{amsthm}
		\usepackage{thmtools}
	\usepackage{acro}
	\acsetup{single}
	\DeclareAcronym{AMCD}{short=AMCD, long={Aide Multicritère à la Décision}}
\DeclareAcronym{AHP}{short=AHP, long={Analytic Hierarchy Process}}
\DeclareAcronym{AR}{short=AR, long={Argumentative Recommender}}
\DeclareAcronym{DA}{short=DA, long={Decision Analysis}}
\DeclareAcronym{DJ}{short=DJ, long={Deliberated Judgment}}
\DeclareAcronym{DM}{short=DM, long={Decision Maker}}
\DeclareAcronym{DP}{short=DP, long={Deliberated Preference}}
\DeclareAcronym{MAVT}{short=MAVT, long={Multiple Attribute Value Theory}}
\DeclareAcronym{MCDA}{short=MCDA, long={Multicriteria Decision Aid}}
\DeclareAcronym{MIP}{short=MIP, long={Mixed Integer Program}}
\DeclareAcronym{SCR}{short=SCR, long={Social Choice Rule}}
\DeclareAcronym{SEU}{short=SEU, long={Subjective Expected Utility}}

\iftoggle{LCpres}{
	\usepackage{fmtcount}
	\usepackage[nodayofweek]{datetime}
}{
}
\makeatletter
\iftoggle{LCpres}{
	\usepackage{hyperref}
}{
	\usepackage[hypertexnames=false, pdfusetitle, linkbordercolor={1 1 1}, citebordercolor={1 1 1}, urlbordercolor={1 1 1}]{hyperref}
	\pdfstringdefDisableCommands{%
		\let\thanks\@gobble
	}
}
\makeatother
	
	\usepackage{bookmark}
	\ExplSyntaxOn
	\seq_new:N \g_oc_hrauthor_seq
	\NewDocumentCommand{\addhrauthor}{m}{
		\seq_gput_right:Nn \g_oc_hrauthor_seq { #1 }
	}
	\NewExpandableDocumentCommand{\hrauthor}{}{
		\seq_use:Nn \g_oc_hrauthor_seq {,~}
	}
	\ExplSyntaxOff
	{
		\catcode`#=11\relax
		\gdef\fixauthor{\xpretocmd{\author}{\addhrauthor{#2}}{}{}}%
	}
	\iftoggle{LCart}{
		\usepackage{authblk}
		
		\fixauthor
		\AtBeginDocument{
		    \hypersetup{pdfauthor={\hrauthor}}
		}
	}{
	}
	\makeatletter
	\g@addto@macro\@floatboxreset\centering
	\makeatother
	\usepackage{doi}
	\usepackage{fixmath}
	\usepackage{amsfonts}
	\usepackage[scr]{rsfso}
	\usepackage{bm}
	\nottoggle{LCpres}{
		\usepackage{cleveref}
	}{
	}
	\nottoggle{LCposter}{
		\expandafter\def\csname ver@etex.sty\endcsname{3000/12/31}
		\usepackage{autonum}
	}{
	}
	\usepackage{xcolor}
\iftoggle{LCpres}{
	\usepackage{tikz}
}{
}
\usepackage{graphicx}
	\graphicspath{{graphics/}}

\iftoggle{LCpres}{
	\usepackage{appendixnumberbeamer}
	\setbeamertemplate{navigation symbols}{} 
	\usepackage{beamerthemeParisFrance}
	\setcounter{tocdepth}{10}
}{
}

\usepackage{xcolor}

\definecolor{ao(english)}{rgb}{0.0, 0.5, 0.0}

\NewDocumentCommand{\possessivecite}{O{}m}{\citeauthor{#2}’s \citeyearpar[#1]{#2}}
\NewDocumentCommand{\Possessivecite}{O{}m}{\Citeauthor{#2}’s \citeyearpar[#1]{#2}}

\iftoggle{LCart}{
	\declaretheorem{theorem}
	\declaretheorem{lemma}
	\declaretheorem{proposition}
	
	\declaretheorem{corollary}
	
	\declaretheorem[style=remark, qed=$\triangle$]{example}
	
}

\newunicodechar{α}{\alpha}
\newunicodechar{β}{\beta}
\newunicodechar{γ}{\gamma}
\newunicodechar{δ}{\delta}
\newunicodechar{ε}{\epsilon}
\newunicodechar{λ}{\lambda}
\newunicodechar{μ}{\mu}
\newunicodechar{ν}{\nu}
\newunicodechar{π}{\pi}
\newunicodechar{τ}{\tau}
\newunicodechar{φ}{\phi}
\newunicodechar{ϕ}{\varphi}
\newunicodechar{Φ}{\Phi}
\newunicodechar{Ψ}{\Psi}
\newunicodechar{ℕ}{\mathbb{N}}
\newunicodechar{ℚ}{\mathbb{Q}}
\newunicodechar{ℝ}{\mathbb{R}}
\newunicodechar{↔}{\leftrightarrow}
\newunicodechar{⇏}{\nRightarrow}
\newunicodechar{⇐}{\ensuremath{\Leftarrow}}
\newunicodechar{⇒}{\ensuremath{\Rightarrow}}
\newunicodechar{⇔}{\Leftrightarrow}
\newunicodechar{⇝}{\rightsquigarrow}
\newunicodechar{∅}{\emptyset}
\newunicodechar{−}{\ifmmode{-}\else\textminus\fi}
\newunicodechar{∞}{\infty}
\newunicodechar{∧}{\land}
\newunicodechar{∨}{\lor}
\newunicodechar{∩}{\cap}
\newunicodechar{∪}{\cup}
\newunicodechar{≠}{\ensuremath{\neq}}
\newunicodechar{≤}{\leq}
\newunicodechar{≥}{\geq}
\newunicodechar{≰}{\nleq}
\newunicodechar{≱}{\ngeq}
\newunicodechar{≲}{\lesssim}
\newunicodechar{≳}{\gtrsim}
\newunicodechar{≴}{\not\preccurlyeq}
\newunicodechar{≵}{\not\succcurlyeq}
\newunicodechar{≹}{\ngtreqlessslant}
\newunicodechar{≺}{\prec}
\newunicodechar{≻}{\succ}
\newunicodechar{≼}{\preccurlyeq}
\newunicodechar{≽}{\succcurlyeq}
\newunicodechar{≾}{\precsim}
\newunicodechar{≿}{\succsim}
\newunicodechar{⊀}{\nprec}
\newunicodechar{⊁}{\nsucc}
\newunicodechar{⊆}{\subseteq}
\newunicodechar{⊢}{\vdash}
\newunicodechar{⊥}{\bot}
\newunicodechar{⊬}{\nvdash}
\newunicodechar{⊲}{\vartriangleleft}
\newunicodechar{⊳}{\vartriangleright}
\newunicodechar{⊴}{\trianglelefteq}
\newunicodechar{⊵}{\trianglerighteq}
\newunicodechar{⋈}{\bowtie}
\newunicodechar{⋪}{\ntriangleleft}
\newunicodechar{⋫}{\ntriangleright}
\newunicodechar{⋬}{\ntrianglelefteq}
\newunicodechar{⋭}{\ntrianglerighteq}
\newunicodechar{□}{\Box}
\newunicodechar{✓}{\checkmark}
\newunicodechar{⟦}{\llbracket}
\newunicodechar{⟧}{\rrbracket}
\newunicodechar{⟼}{\longmapsto}
\newunicodechar{⪰}{\succeq}
\newunicodechar{〚}{\textlbrackdbl}
\newunicodechar{〛}{\textrbrackdbl}
\DeclareUnicodeCharacter{00B1}{\pm}
\DeclareUnicodeCharacter{2190}{\ifmmode\leftarrow\else\textleftarrow\fi}
\DeclareUnicodeCharacter{2192}{\ifmmode\rightarrow\else\textrightarrow\fi}
\DeclareUnicodeCharacter{21D2}{\Rightarrow}
\DeclareUnicodeCharacter{00AC}{\ifmmode\lnot\else\textlnot\fi}
\DeclareUnicodeCharacter{2026}{\ifmmode\dots\else\textellipsis\fi}
\DeclareUnicodeCharacter{00D7}{\ifmmode\times\else\texttimes\fi}
\catcode`\'=\active
\DeclareUnicodeCharacter{0027}{\ifmmode^\prime\else\textquotesingle\fi}

\NewDocumentCommand{\suchthat}{}{\;\ifnum\currentgrouptype=16 \middle\fi|\;}

\NewDocumentCommand{\knowing}{}{\;\ifnum\currentgrouptype=16 \middle\fi|\;}
\NewDocumentCommand{\intvl}{m}{⟦#1⟧}

\DeclarePairedDelimiter\set{\{}{\}}
\DeclareMathOperator*{\argmax}{arg\,max}

\let\phi\varphi

\NewDocumentCommand{\classifs}{}{\mathcal{C}}
\NewDocumentCommand{\cn}{}{\bm{c}}

\NewDocumentCommand{\prefi}{O{i}}{\succ_{#1}}

\NewDocumentCommand{\mleadsto}{O{\eta}}{⇝_{#1}}
\NewDocumentCommand{\mbeats}{O{\eta}}{⊳_{#1}}

\title{Classification aggregation without unanimity\thanks{Corresponding author:\\ Matthieu Hervouin\\
E-mail: \href{mailto:matthieu.hervouin@dauphine.fr}{\texttt{matthieu.hervouin@dauphine.eu}}\\
Postal Adress: Université Paris Dauphine, LAMSADE, Matthieu Hervouin, Place du Maréchal de Lattre de Tassigny, 75016, Paris, France}}
\author[1]{Olivier Cailloux}
\author[1]{Matthieu Hervouin}
\author[2,1]{Ali I. Ozkes}
\author[1]{M. Remzi Sanver}
\affil[1]{Université Paris-Dauphine, Université PSL, CNRS, LAMSADE, 75016 Paris, France}
\affil[2]{SKEMA Business School, Université Côte d’Azur, GREDEG, Paris, France.}

\hypersetup{
	pdfsubject={},
	pdfkeywords={aggregation, independence, categories},
}

\begin{document}
\maketitle

\begin{abstract}
  A classification is a surjective mapping from a set of objects to a set of categories. A classification aggregation function aggregates every vector of classifications into a single one. We show that every citizen sovereign and independent classification aggregation function is essentially a dictatorship. This impossibility implies an earlier result of Maniquet and Mongin (2016), who show that every unanimous and independent classification aggregation function is a dictatorship. The relationship between the two impossibilities is reminiscent to the relationship between Wilson’s and Arrow’s impossibilities in preference aggregation. Moreover, while the Maniquet-Mongin impossibility rests on the existence of at least three categories, we propose an alternative proof technique that covers the case of two categories, except when the number of objects is also two. We also identify all independent and unanimous classification aggregation functions for the case of two categories and two objects.
\end{abstract}

\section{Introduction}
A relatively recent aggregation impossibility in social choice theory has been established by \citet{maniquet2016theorem}
within the context of aggregating classifications. Their model is constructed out of a finite set of at least three categories, a finite set of objects, and a finite set of individuals who classify objects into categories. A classification is a surjective mapping from the set of objects to the set of categories, \emph{i.e.}, to every category at least one object is classified. The number of considered objects is no less than the number of available categories, which allows to assume surjectivity. A classification aggregation function maps a given vector (or profile) of classifications (one for each individual) into a single classification.

The model incorporates two other conditions. One is unanimity: if all individuals make the same classification then the aggregate classification complies. The other condition, independence, requires an object to be classified identically by the classification aggregator at any two classification profiles where every individual classifies this object identically. The model and its conditions are Arrovian in spirit, which ends up in an Arrovian impossibility, expressed by Theorem 1 in \citet{maniquet2016theorem}: every independent and unanimous classification aggregation function is a dictatorship. We refer to this result as the MM (\citeauthor{maniquet2016theorem}) impossibility.

We allow the number of categories to be at least two but require the number of objects exceed the number of categories. We replace unanimity with a weaker condition, namely citizen sovereignty, which requires that any object can be classified in any category. We prove that every independent and citizen sovereign classification aggregation function is essentially a dictatorship. Citizen sovereignty is implied by unanimity. An essential dictatorship is more general than a dictatorship but the two concepts are equivalent under unanimity. Thus, our result implies the MM impossibility within our environment. The case of two categories and two objects is the only one that escapes both the MM impossibility and our impossibility. We treat this case separately and identify all classification aggregation functions that are unanimous and independent.

Section \ref{Sec:prel} gives the basic notions and notation. Section \ref{proofs} presents the results. Section \ref{concl} makes some concluding remarks.

\section{Basic notions and notation}
\label{Sec:prel}
We consider a set $N = \set{1, …, n}$ of \emph{individuals} with $n \geq 2$, a set $P = \set{p_1, …, p_\rho}$ of \emph{categories} with $\rho \geq 2$, and a set $X = \set{x_1, …, x_m}$ of \emph{objects} with $m \geq \rho$.

We define a \emph{classification} as a surjective mapping $c: X \rightarrow P$ and denote by $\mathcal{C} \subset P^X$ the set of classifications.
We write $ \bm{c}=(c_1,\ldots ,c_n) \in \mathcal{C}^N$ for a classification profile. Given $ \bm{c} \in \mathcal{C}^N$ and $x\in X$, we write $\bm{c}_x \in P^N$ for the vector of categories that object $x$ is put into by each individual, thus $\forall i \in N, \bm{c}_x(i)=c_i(x)$.
A \emph{classification aggregation function} (CAF) is a mapping $\alpha: \mathcal{C}^N \rightarrow \mathcal{C}$.

A CAF $\alpha$ is \emph{independent} if $ \forall x \in X, \forall \bm{c}, \bm{c}' \in \classifs^N, \bm{c}_x = \bm{c}'_x ⇒ α(\bm{c})(x) = α(\bm{c}')(x)$.

An \emph{elementary} CAF is a mapping from $P^N$ to $P$. Note that an independent CAF $\alpha$ can be expressed as a collection $(\alpha_x)_{x\in X}$ of elementary CAFs with $\alpha_x(\cn_x) = \alpha(\cn)(x)$.

A CAF $\alpha$ is a \emph{dictatorship} if there exists $d \in N$ such that for all $\bm{c} \in \mathcal{C}^N, \alpha(\bm{c})=c_d$. Let $\pi : P ↔ P$ denote a permutation of $P$. A CAF is \emph{essentially a dictatorship} if there exists $ d \in N,$ and a permutation $\pi$ of $P$ such that for all $\bm{c} \in \mathcal{C}^N, \alpha(\bm{c})=\pi \circ c_d$. Every dictatorship is essentially a dictatorship.

A CAF $\alpha$ is \emph{unanimous} if $\forall c \in \mathcal{C}: \alpha(c, \ldots ,c)=c$. A CAF $\alpha$ is \emph{citizen sovereign} if $\forall x \in X, p \in P, \exists \bm{c} \in \mathcal{C}^N$ such that $ \alpha(\bm{c})(x)=p$. Unanimity implies citizen sovereignty.

We exemplify a CAF and illustrate some of these concepts.

\begin{example}[$\alpha_\mathit{PLUR}$]\label{example}
  Let $N=\set{1,2,3}, X=\set{x,y,z}$ and $P=\set{p,q}$. We define $\alpha_{PLUR}$ as the CAF that selects the classification that is the plurality winner, breaking ties according to an exogenous linear order $T$ on $\mathcal{C}$. Thus, for $\bm{c}\in \mathcal{C}^N, \alpha_\mathit{PLUR}(\bm{c})= \max\limits_T( \argmax\limits_{c \in \mathcal{C}} |\set{i \in N \suchthat c_i=c}|)$.
  One can check that $\alpha_{PLUR}$ is a CAF and is unanimous while not being essentially a dictatorship.

  Let the classification that maps $x $ to $p$ and $y,z$ to $q$ be the maximal element of $T$.
  \Cref{plur} shows the behavior of $\alpha_\mathit{PLUR}$ for two classification profiles. We can see that $\alpha_\mathit{PLUR}$ does not satisfy independence as the profiles are the same regarding object $x$, yet $\alpha_\mathit{PLUR}(\bm{c})(x)\neq \alpha_\mathit{PLUR}(\bm{c}')(x)$.

  \begin{table}\caption{The behaviour of $\alpha_\mathit{PLUR}$ on two classification profiles.}
    \label{plur}

    $\begin{array}{ccccccccc}\toprule
                      & \multicolumn{3}{c}{\bm{c}} &   & \multicolumn{3}{c}{\bm{c}'} &                                                                          \\
        \mbox{object} & 1                          & 2 & 3                           & \alpha_\mathit{PLUR}(\bm{c}) & 1 & 2 & 3 & \alpha_\mathit{PLUR}(\bm{c}') \\ \midrule
        x             & p                          & q & q                           & p                            & p & q & q & q                             \\
        y             & q                          & p & q                           & q                            & q & q & q & q                             \\
        z             & q                          & q & p                           & q                            & q & p & p & p                             \\ \bottomrule
      \end{array}$

  \end{table}

\end{example}

\section{Results}\label{proofs}
\begin{theorem}\label{thm1}
  For $m>\rho \geq 2$, every citizen sovereign and independent CAF is essentially a dictatorship.
\end{theorem}
We prove Theorem \ref{thm1} by conjoining Lemmata \ref{lem: CS} and \ref{lem: GU} stated and proven below. We also define a generalization of unanimity that we need in our proofs: a CAF $\alpha$ satisfies \emph{generalized unanimity} (GU) if $\exists \pi : P \leftrightarrow P \suchthat \forall c \in \mathcal{C}, \alpha(c,\ldots ,c)=\pi(c)$.

\begin{lemma}\label{lem: CS}
  For $m>\rho \geq 2$, every citizen sovereign and independent CAF satisfies GU.
\end{lemma}

\begin{proof}
  Let $\alpha=(\alpha_x)_{x \in X}$ be a CAF that is citizen sovereign and independent. Define $\pi: P \rightarrow P$ as $ \pi(p_i) = \alpha_{x_i}(p_i, \ldots, p_i), \forall i \in \intvl{1, \rho}$. We prove the lemma by showing that $\pi $ is a bijection over $P$ satisfying $\forall p \in P,  x \in X: \alpha_x(p, \ldots, p)=\pi(p)$. First, we show that $\pi $ is a bijection.

  For any $q \in P$ and any $y \in X$, citizen sovereignty and independence ensure the existence of $\bm{k}_{y,q} \in P^N$ such that $\alpha_y(\bm{k}_{y,q})= q$.

  Given $r \in P$, we define the classification profile $\cn^{(r)}$ as $\forall i \in \intvl{1,\rho}, \bm{c}^{(r)}_{x_i}= (p_i, \ldots, p_i )$ (note that it does not depend on $r$) and $\forall l \in \intvl{ \rho+1, m}, \cn^{(r)}_{x_l}= \bm{k}_{x_l,r}$. The classification profile $\cn^{(r)}$ is written in table form below. For the rest of this proof we will describe profiles using this form only, to ease readability.

  \[ \begin{array}{ccc}\toprule
                                   & \cn^{(r)}        &                   \\
      \mbox{object}                & 1, \ldots , n    & \alpha(\cn^{(r)}) \\ \midrule
      x_1                          & p_1, \ldots, p_1 & \pi(p_1)          \\
      x_i, i \in \intvl{2, \rho}   & p_i, \ldots, p_i & \pi(p_i)          \\
      x_l, l \in \intvl{\rho +1,m} & \bm{k}_{x_l,r}   & r                 \\ \bottomrule
    \end{array}\]

  Observe that each $\bm{c}_i^{(r)}$, $i \in N$, is surjective as required.
  As $\alpha(\bm{c}^{(r)})$ must be surjective,
  and $\forall l \in \intvl{\rho+1,m} :\alpha_{x_l}(\cn^{(r)}_{x_l}) = \alpha_{x_l}(\bm{k}_{x_l, r}) = r$ (by definition of $\bm{k}_{x_l, r}$),
  we must have  $\{\alpha_{x_i}( \bm{c}^{(r)}_{x_i}), i \in \intvl{1,\rho}\} \supseteq P \setminus \{r\}$ and therefore $\{ \pi(p_i), i \in \intvl{1, \rho} \} \supseteq P\setminus\{ r \}$.

  As this holds for all $r \in P$, in particular, $\{ \pi(p_i), i \in \intvl{1, \rho} \} \supseteq P\setminus\{ p_1 \}$ and $\{ \pi(p_i), i \in \intvl{1, \rho} \} \supseteq P\setminus\{ p_2 \}$, whence $\{\pi(p_i), i \in \intvl{1, \rho} \} = P$. Therefore, $\pi(P)=P$, so $\pi $ is bijective.

  Now, considering any $i \in \intvl{1, \rho}$, $j \in \intvl{\rho +1, m}$, we prove that $\alpha_{x_j}(p_i, \ldots, p_i)=\pi(p_i)$.
  Take any $r \in P \setminus \{\pi(p_i)\}$, and consider the classification profile $\bm{c}^{(r)(i)}$ defined in the following table.

  \[ \begin{array}{ccc}\toprule
                                                 & \cn^{(r)(i)}     &                      \\
      \mbox{object}                              & 1,\ldots,n       & \alpha(\cn^{(r)(i)}) \\ \midrule
      x_i                                        & \bm{k}_{x_i,r}   & r                    \\
      x_l, l \in \intvl{1,\rho}\setminus \set{i} & p_l, \ldots, p_l & \pi(p_l)             \\
      x_j                                        & p_i, \ldots, p_i &                      \\
      x_l, l \in \intvl{\rho +1,m}               & \bm{k}_{x_l,r}   & r                    \\ \bottomrule
    \end{array} \]

  Observe that each $\bm{c}_l^{(r)(i)}$, $l \in N$, is surjective as required.
  We have that $\forall l \in \intvl{1,\rho}\setminus \{i\}, \alpha_{x_l}(\bm{c}^{(r)(i)}_{x_l})=\alpha_{x_l}(p_l, \ldots, p_l)=\pi(p_l) \neq \pi(p_i), \forall l \in \intvl{\rho +1, m}\setminus \{j\}, \alpha_{x_l}(\bm{c}^{(r)(i)}_{x_l})=\alpha_{x_l}(\bm{k}_{x_l,r})= r \neq \pi(p_i)$, and $\alpha_{x_i}(\bm{c}^{(r)(i)}_{x_i})=\alpha_{x_i}(\bm{k}_{x_i,r})= r$. Then, we have that $\{\alpha_{x_l}(\bm{c}^{(r)(i)}_{x_l}), l \in N \setminus \{j\}  \}= P \setminus \{ \pi (p_i)\}$, and as $\alpha(\bm{c}^{(r)(i)})$ must be surjective, we must have $\alpha_{x_j}(\bm{c}^{(r)(i)}_{x_j})=\pi(p_i)$. Therefore, $\alpha_{x_j}(p_i, \ldots , p_i)=\pi(p_i)$.

  Remains only to prove that considering any $i\neq j \in \intvl{1, \rho}$: $\alpha_{x_j}(p_i, \ldots, p_i)= \pi(p_i)$. Take any $r \in P \setminus\{\pi(p_i) \}$, and consider the classification profile $\bm{c}^{(r)(i,j)}$ defined in the following table.

  \[\begin{array}{ccc}\toprule
                                                   & \cn^{(r)(i,j)}   &                      \\
      \mbox{object}                                & 1, \ldots, n     & \alpha(\cn^{(r)(i)}) \\ \midrule
      x_i                                          & \bm{k}_{x_i,r}   & r                    \\
      x_j                                          & p_i, \ldots, p_i &                      \\
      x_l, l \in \intvl{1,\rho}\setminus \set{i,j} & p_l, \ldots, p_l & \pi(p_l)             \\
      x_{\rho+1}                                   & p_j, \ldots, p_j & \pi(p_j)             \\
      x_l, l \in \intvl{\rho +2,m}                 & \bm{k}_{x_l,r}   & r                    \\ \bottomrule
    \end{array} \]

  Observe that each $\bm{c}_l^{(r)(i,j)}$, $l \in N$, is surjective as required. By definition of $\bm{k}_{x_l,r}$, $\forall l \in \intvl{\rho +2, m}, \alpha_{x_l}(\bm{c}^{(r)(i,j)}_{x_l})=r \neq \pi(p_i)$, and from the previous result, $\alpha_{x_{\rho+1}}(\bm{c}^{(r)(i,j)}_{x_{\rho+1}})=\pi(p_j)\neq \pi(p_i)$. Finally, $ \forall l \in  \intvl{1,  \rho} \setminus \{i,j\}, \alpha_{x_l}(\bm{c}^{(r)(i,j)}_{x_l})=\pi(p_{l})\neq \pi(p_i)$. Then, we
  have $\{ \alpha_{x_l}(\bm{c}^{(r)(i,j)}_{x_l}), l \in N \setminus \{j\}   \}= P \setminus \{ \pi(p_i) \}$. By surjectivity of $\alpha(\bm{c}^{(r)(i,j)})$, we must have $\alpha_{x_j}(\bm{c}^{(r)(i,j)}_{x_j})=\alpha_{x_j}(p_i, \ldots, p_i)= \pi(p_i)$.

  Therefore, we proved that $\exists \pi : P \leftrightarrow P$ such that $\forall i \in \intvl{1, \rho} , j \in \intvl{1, m}: \alpha_{x_j}(p_i, \ldots, p_i)= \pi(p_i)$. Thus, $\alpha=(\alpha_x)_{x \in X}$ satisfies GU.
\end{proof}

\begin{lemma}\label{lem: GU}
  For $m>\rho \geq 2$, every CAF that satisfies independence and GU is essentially a dictatorship.
\end{lemma}

\begin{proof}

  Let $\alpha$ be a CAF that satisfies independence and GU. Consider the permutation $\pi : P \leftrightarrow P$ such that $\forall p \in P, x \in X: \pi(p)=\alpha_x(p, \ldots, p)$. (By GU, that permutation exists.)

  First, let us show that $\forall x ≠ y \in X, p, q \in P, r, r' \in P^N: [\forall i \in N, \{r_i, r_i'  \}=\{p,q\}] ⇒ \set{\alpha_x(r), \alpha_y(r')} = \set{\pi(p), \pi(q)}$. (1)

  Consider any $r, r' \in P^N \suchthat \forall i \in N: \{r_i, r_i'  \}=\{p,q\}$. If $p = q$, the claim holds by definition of $\pi$,
  and if $\rho = 2, p ≠ q$, the claim holds by considering the profile $\cn_x = r$ and $\forall z \in X \setminus \set{x}, \cn_z = r'$ and using the surjectivity of $α(\cn)$. Thus, assume $p ≠ q$, $\rho ≥ 3$. Consider, wlog, that $p=p_1, q=p_2, x = x_1$ and $y = x_2$. Define the classification profile $\cn$ as $\cn_{x_1} = r$, $\cn_{x_2} = r'$, $\forall 3 ≤ l ≤ \rho: \cn_{x_l} = (p_l, \ldots, p_l)$ and $\forall \rho \leq l \leq m, \cn_{x_l}=(p_\rho, \ldots, p_\rho)$. 

  \[\begin{array}{ccc}\toprule
                                    & \cn            &             \\
      \mbox{object}                 & 1\leq k \leq n & \alpha(\cn) \\ \midrule
      x_1                           & r_k            &             \\
      x_2                           & r'_k           &             \\
      x_l, l \in \intvl{3,\rho}     & p_l            & \pi(p_l)    \\
      x_l, l\in \intvl{ \rho +1, m} & p_\rho         & \pi(p_\rho) \\ \bottomrule
    \end{array}\]

  One can check that $\forall i \in N$, the resulting $c_i$ is surjective. By definition of $\pi, \forall l \in \intvl{3, \rho}, \alpha_{x_l}(\cn)=\pi(p_l) \notin \set{p_1, p_2}$, and $\forall l \in \intvl{\rho +1, m}, \alpha_{x_l}(\cn)= \pi(p_\rho)\notin \set{\pi(p_1, \pi(p_2)}$.
  Because $\alpha(\cn)$ must be surjective, $\set{\pi(p_1), \pi(p_2)} = \set{\alpha_{x_1}(\cn_{x_1}), \alpha_{x_2}(\cn_{x_2})}$.

  Now, we show that $\forall p\in P, x\in X, \exists d \in N \text{ such that } \forall t \in P^N, t(d)=p ⇒ \alpha_x(t)=\pi(p)$. (2)

  Suppose wlog that $p=p_1$ and $x=x_1$. Given any $0 ≤ i ≤ n$, define $r^i, l^i \in P^N$ as $\forall j \in \intvl{1,i}, r^i_j=p_2, l^i_j=p_1$ and $\forall j \in \intvl{i+1,n}, r^i_j=p_1, l^i_j=p_2$.

  Define $d = \min \set{i \in N \suchthat \alpha_{x_3}(r^i) = \pi(p_2)}$ (such a $d$ exists as $\alpha_{x_3}(r^n)= \pi(p_2)$ by definition of $\pi$), consider any $t \in P^N \suchthat t(d) = p_1$, and let us show that $\alpha_{x_1}(t) = \pi(p_1)$.

  First observe that by (1), $\set{\alpha_{x_3}(r^{d - 1}), \alpha_{x_2}(l^{d - 1})} = \set{\pi(p_1), \pi(p_2)}$.
  Also, $\alpha_{x_3}(r^{d - 1}) ≠ \pi(p_2)$: if $d ≥ 2$ it follows from the minimality of $d$, and if $d = 1$ it follows from $\alpha_{x_3}(p_1, \ldots, p_1) = \pi(p_1)$ by definition of $\pi$.
  Therefore, $\alpha_{x_2}(l^{d - 1}) = \pi(p_2)$.

  Define a classification $\cn'$ as $\cn'_{x_1} = t$, $\cn'_{x_2} = l^{d - 1}$, $\cn'_{x_3} = r^d$, $\forall 4 ≤ k ≤ m: \cn'_{x_k} = (p_{\min \set{k - 1, \rho}}, \ldots, p_{\min \set{k - 1, \rho}})$.

  \[\begin{array}{cccccc}\toprule
                                    & \multicolumn{3}{c}{\cn'} &                                      \\
      \mbox{object}                 & 1\leq i <d               & d       & d< j \leq n & \alpha(\cn') \\ \midrule
      x_1                           & t_i                      & p_1     & t_j         &              \\
      x_2                           & p_1                      & p_2     & p_2         & \pi(p_2)     \\
      x_3                           & p_2                      & p_2     & p_1         & \pi(p_2)     \\
      x_{k}, k\in \intvl{4, \rho}   & p_{k-1}                  & p_{k-1} & p_{k-1}     & \pi(p_{l-1}) \\
      x_k, k\in \intvl{ \rho +1, m} & p_\rho                   & p_\rho  & p_\rho      & \pi(p_\rho)  \\ \bottomrule
    \end{array}\]

  Observe that $\forall i \in N$, the resulting $c'_i$ is surjective.
  By definition of $\pi$, $\forall 4 ≤ k ≤ m: \alpha_{x_k}(\cn'_{x_k}) = \pi(p_{\min \set{k - 1, \rho}})$.
  Because $\alpha(\cn')$ must be surjective and $\alpha_{x_2}(\cn'_{x_2}) = \alpha_{x_3}(\cn'_{x_3}) = \pi(p_2)$, we obtain $\alpha_{x_1}(\cn'_{x_1}) = \pi(p_1)$.

  We have just established (2), we now aim to show that $d$ is the same for all categories, thus, that $\forall x \in X, \exists d \in N \suchthat \forall t \in P^N: \alpha_x(t) = \pi(t(d))$. (3)

  Fixing $x \in X$, and considering any $p, q \in P$, by (2), $\exists d_p, d_q \in N \suchthat \forall t \in P^N, t(d_p) = p \land t(d_q) = q ⇒ \alpha_x(t)=π(t(d_p)) = π(t(d_q))$.
  If $d_p ≠ d_q$ (thus $q ≠ p$), then pick some $t$ such that $t(d_p) = p$ and $t(d_q) = q$ and obtain the incoherent $\alpha_x(t) = \pi(p)$ and $\alpha_x(t) = \pi(q)$ with $\pi(p) ≠ \pi(q)$.
  Thus, $d_p=d_q$. This establishes (3).

  Now, let us prove that the decisive individuals are equal across the objects.
  Picking $x ≠ y \in X$ and $p ≠ q \in P$, by (3), $\exists d_x, d_y \in N \suchthat \forall t \in P^N, \alpha_x(t)=\pi(t(d_x))$ and $\alpha_y(t)=\pi(t(d_y))$. Define $t$ as $t_{d_x}=p$ and, $ \forall j \in N \setminus\set{d_x}, t_j=q$. Define $t'$ as $ t'_{d_x}=q$ and, $ \forall j \in N \setminus\set{d_x}, t'_j=p$. We have $\alpha_x(t) = \pi(t(d_x)) = \pi(p)$ and $\alpha_{y}(t') = \pi(t'(d_y))$. Apply (1) to obtain $\set{\alpha_x(t), \alpha_{y}(t')} = \set{\pi(p), \pi(q)}$.
  Thus, $\pi(t'(d_y)) = \pi(q)$, so $t'(d_y)=q$ which implies $d_y = d_x$ (as $\forall j ≠ d_x: t'(j) = p$).

  We have shown that $\exists d \in N \suchthat \forall x \in X, \forall t \in P^N: \alpha_x(t)=\pi(t(d))$, so there exists an individual $d$ which is essentially a dictator.
\end{proof}

\begin{corollary}\label{coro1}
  For $m >\rho \geq 2$, every unanimous and independent CAF $\alpha = (\alpha_x)_{x \in X}$ is a dictatorship.
\end{corollary}
\begin{proof}
  Let $\alpha$ be a CAF that satisfies unanimity and independence. Thus, $\alpha$ is citizen sovereign and by \cref{thm1}, $\alpha$ is essentially a dictatorship. Hence, there exists a permutation $\pi$ of $P$ and an individual $d \in N$ such that $\forall x \in X, \bm{c} \in \mathcal{C}^N$, we have $\alpha_x(\bm{c})=\pi(c_d(x))$. As $\alpha$ is unanimous, $\pi$ must be the identity function. Therefore, $\forall x \in X, \bm{c} \in \mathcal{C}^N \alpha_x(\bm{c})=c_d(x)$, so $\alpha$ is a dictatorship.
\end{proof}

We quote below the MM impossibility.

\begin{theorem}[\citet{maniquet2016theorem}] \label{thm2}

  For $m \geq \rho>2$, every unanimous and independent CAF is a dictatorship.
\end{theorem}
Combining Theorem \ref{thm2} and Corollary \ref{coro1} lead to the following corollary.
\begin{corollary}\label{corol}
  For $m\geq \rho \geq 2$ and $m\geq 3$, every unanimous and independent CAF is a dictatorship.
\end{corollary}
The impossibilities we show do not cover the case $m=\rho=2$ where, in fact, there exist unanimous and independent CAFs that are not essentially dictatorships. To identify these, let $X=\{x,y\}$ and $P=\{p,q\}$. Write $\overline{p}=q $ and $\overline{q}=p$ as well as $\overline{\bm{r}}=(\overline{r_1}, \ldots \overline{r_n})$ for $\bm{r} \in P^N$.
\begin{proposition}
  Given $X=\{x,y\}$, $P=\{p,q\}$, and any unanimous elementary CAF $\alpha_x: P^N\rightarrow P$, an elementary CAF $\alpha_y: P^N \rightarrow P$ induces a unanimous and independent CAF $(\alpha_x, \alpha_y)$ iff we have $\alpha_y(\bm{r}) = \overline{\alpha_x(\overline{\bm{r}})}$ for all $\bm{r}\in P^N.$
\end{proposition}
\begin{proof}
  Let $\alpha_x$ be a unanimous elementary CAF for $x$, thus $\alpha_x(p, \ldots, p)=p$ and $\alpha_x(q, \ldots, q)=q$.
  We define an elementary CAF on $y$, $\alpha_y$, as follows.
  For any profile $\bm{r} \in P^N$ we have $\alpha_y(\overline{\bm{r}}) = p \mbox{ iff } \alpha_x(\bm{r}) = q$. Doing so, $x$ and $y$ are always put in different categories, and surjectivity is guaranteed.
  Moreover, by definition, $\alpha_x$ satisfies surjectivity and $\alpha_y(p, \ldots, p)=\overline{\alpha_x(\overline{p}, \ldots, \overline{p})}=\overline{\alpha_x(q, \ldots, q)}=p$ so that $\alpha_y(q, \ldots, q)=q$ by the same reasoning.
  Thus, $(\alpha_x,\alpha_y)$ is a unanimous CAF, which is also independent as both $\alpha_x$ and $\alpha_y$ can be defined independently.

  Finally, given $\alpha_x$, if $\alpha_y$ does not satisfy $\alpha_y(\overline{\bm{r}}) = p \mbox{ iff } \alpha_x(\bm{r}) = q$ for all $\bm{r}\in P^N$, then the CAF $(\alpha_x,\alpha_y)$ fails surjectivity, so there is a unique $\alpha_y$ that makes $(\alpha_x,\alpha_y)$ a CAF for a given $\alpha_x$.
\end{proof}

\section{Concluding remarks}\label{concl}
The MM impossibility is the reflection of Arrow’s impossibility theorem to the classification aggregation problem. Does the theorem of \cite{wilson1972social} also have such a reflection? Our \cref{thm1} answers affirmatively. As a matter of fact, unanimity plays a minor role in establishing the MM impossibility, which is essentially a tension between independence and the surjectivity of classifications.

Our proof technique differs from that of the MM impossibility, which is proven through ultrafilters where decisive coalitions eventually shrink to a singleton, namely, the dictator. Our approach follows an alternative proof technique exploited by \citet{yu2012one}, where an individual who is pivotal at some instance is shown to be pivotal whenever she has an incentive to change the outcome, thus being the dictator. Our proof does not only deliver a compact method; it also highlights the impact of surjectivity and independence.
Moreover, the pivotal voter technique is able to cover the case of two categories that is excluded by the MM impossibility, although it should be noted that it does not cover the case $m = \rho$.

We recall that the Arrovian impossibility in preference aggregation is valid for at least three alternatives while it vanishes for the case of precisely two alternatives, where there are several interesting aggregation rules (see \citet{llamazares2013structure} and \citet{ozkes2017absolute}, among others). Given the Arrovian spirit of the classification aggregation model, it is natural to ask whether a similar picture emerges for the classification aggregation problem. Thus, covering the case of two categories has its own merit. Our finding presents a rather different phenomenon compared to the Arrovian aggregation problem in that the impossibility does not vanish in the case of two categories. This difference can be explained by contrasting the binary nature of the Arrovian independence that entails a trivial satisfaction for two alternatives with the unary nature of the MM independence that makes this condition still demanding for two categories.

The classification aggregation problem has an interesting connection to the group identification problem introduced by \citet{kasher1997question}. In group identification, the individuals are to be classified into categories, which makes it a classification aggregation problem where the set of individuals and the set of objects coincide. The original group identification setting imposes surjectivity, but more recent papers like \citet{sung2005axiomatic} and \citet{fioravanti2021alternative} do not. With surjectivity, we find ourselves with an almost equivalent of the MM impossibility which was first shown for two categories by \citet{kasher1997question}.\footnote{The qualification “almost” is needed because in \citet{kasher1997question} $N$ and $X$ are taken to be the same set, imposing the restriction $n=m$, which we do not have. Moreover, although not explicitly stated, they must be assuming $n \geq 3$ because their Theorem 2 does not hold for $n=m=2$, as noted by our Proposition 1.}  Their proof refers to \citet{rubinstein1986algebraic} whose Theorem 3 shows for a general setting the dictatoriality of independent and unanimous aggregators under an assumption that resembles surjectivity. Nevertheless, this result does not seem to cover the case of more than two categories.

Another connection of interest, noted by \citet{maniquet2014judgment}, is the possibility to embed the classification aggregation problem into the judgement aggregation model. An example is the setting studied by \citet{dokow2010aggregation} who consider judgement aggregation with non-binary values. \citet{maniquet2014judgment} use a theorem from \citet{dokow2010aggregation} to prove a different version of the MM impossibility.\footnote{The difference comes from the restriction on the sizes of $P$ and $X$.} Another example is \citet{dietrich2015aggregation} who studies judgement aggregation on binary evaluations while adopting a new version of the independence property. This setting also embeds classification aggregation and an equivalent version of the MM impossibility is proven. Finally, as a further generalization, \citet{endriss2018graph} explore a model of graph aggregation where they obtain an impossibility that entails the MM impossibility. Whether our impossibility without unanimity prevails in these more general frameworks remains an open question.

We close by mentioning two possible directions to escape the MM impossibility. One
is to consider domain restrictions. The only research we know in this direction is \citet{craven2023domain} who shows the existence of a unanimous and independent CAF, \emph{i.e.}, a version of the majority rule, that is defined over a restricted domain. The other direction is to do away with the independence condition in search of interesting CAFs. This has been done in the context of group identification by \citet{kasher1997question}, who define the strong liberal rule, whose variations are introduced by \citet{sung2005axiomatic}.\footnote{\citet{fioravanti2021alternative} added the inclusive and the unanimous aggregator to the collection.} Within the classification aggregation framework, the aggregator $\alpha_\mathit{PLUR}$ exemplifies a CAF that is unanimous and not independent.

\bibliography{class_agg.bib}

\newpage
\appendix

\end{document}